\documentclass{article}

\usepackage[preprint]{neurips_2026}

\usepackage[utf8]{inputenc}
\usepackage[T1]{fontenc}
\usepackage{microtype}
\usepackage{graphicx}
\usepackage{subcaption}
\usepackage{booktabs}
\usepackage{enumitem}
\usepackage{xcolor}
\usepackage{amsmath}
\usepackage{amssymb}
\usepackage{mathtools}
\usepackage{amsthm}
\usepackage{bbm}
\usepackage{hyperref}
\usepackage{tikz}
\usetikzlibrary{arrows.meta,positioning,shapes.geometric,calc,fit,backgrounds,decorations.pathmorphing,decorations.pathreplacing}

\definecolor{aceSlate}{RGB}{226,232,240}
\definecolor{aceSlateDk}{RGB}{71,85,105}

\definecolor{aceLeaf}{RGB}{198,246,213}
\definecolor{aceLeafDk}{RGB}{34,134,58}

\definecolor{aceRose}{RGB}{255,214,224}
\definecolor{aceRoseDk}{RGB}{192,16,72}

\definecolor{aceAmber}{RGB}{254,215,170}
\definecolor{aceAmberDk}{RGB}{194,65,12}

\definecolor{aceGold}{RGB}{254,243,189}
\definecolor{aceGoldDk}{RGB}{161,98,7}

\definecolor{aceRuby}{RGB}{254,205,211}
\definecolor{aceRubyDk}{RGB}{190,18,60}

\definecolor{aceSky}{RGB}{190,227,248}
\definecolor{aceSkyDk}{RGB}{30,64,175}

\definecolor{aceVio}{RGB}{233,216,253}
\definecolor{aceVioDk}{RGB}{107,33,168}

\definecolor{aceTeal}{RGB}{207,232,231}
\definecolor{aceTealDk}{RGB}{15,118,110}

\definecolor{aceMuted}{RGB}{100,116,139}

\usepackage[capitalize,noabbrev]{cleveref}

\DeclareMathOperator*{\argmax}{arg\,max}

\theoremstyle{plain}
\newtheorem{theorem}{Theorem}[section]
\newtheorem{proposition}[theorem]{Proposition}

\theoremstyle{definition}

\newtheorem{assumption}[theorem]{Assumption}
\theoremstyle{remark}

\title{Active Causal Experimentalist (ACE):\\Learning Intervention Strategies via Direct Preference Optimization}

\author{%
  Patrick Cooper\\
  University of Colorado Boulder\\
  \texttt{patrick.cooper@colorado.edu}\\
  \And
  Alvaro Velasquez\\
  University of Colorado Boulder\\
  \texttt{alvaro.velasquez@colorado.edu}\\
}

\begin{document}
\maketitle

\begin{abstract}
Discovering causal relationships requires controlled experiments, but experimentalists face a sequential decision problem: each intervention reveals information that should inform what to try next. Traditional approaches such as random sampling, greedy information maximization, and round-robin coverage treat each decision in isolation, unable to learn adaptive strategies from experience. We propose Active Causal Experimentalist (ACE), which learns experimental design as a sequential policy. Our key insight is that while absolute information gains diminish as knowledge accumulates, making value-based reinforcement learning unstable, \emph{relative} comparisons between candidate interventions remain meaningful throughout. ACE exploits this via Direct Preference Optimization (DPO), learning from pairwise intervention comparisons rather than non-stationary reward magnitudes. We formalize this advantage with a scale-invariance result (Proposition~\ref{prop:scale-invariance}) and validate it empirically. Across synthetic structural causal models of increasing complexity (5 and 30 nodes), coupled Duffing oscillators, and economic time-series data, ACE achieves a median total loss of $0.61$ on the 5-node benchmark versus $1.99$ (median; $2.04\pm0.12$ mean) for a Bayesian optimal experimental design baseline at matched intervention budgets, a $69\%$ improvement. On the 30-node benchmark, ACE achieves $1.95\pm0.77$ best-loss versus $5.80\pm0.05$ for random intervention, $5.79\pm0.06$ for round-robin, and $5.80\pm0.04$ for max-variance, a $3.0\times$ improvement over all three baselines that converge to a common plateau. Notably, the learned policy autonomously discovers that collider mechanisms require concentrated interventions on parent variables, a theoretically-grounded strategy that emerges purely from experience.
\end{abstract}

\section{Introduction}

Every experimentalist faces limited resources to explore vast possibility spaces. Testing all pairwise combinations of 100 compounds requires 4{,}950 experiments; a 10-component alloy across 5 temperatures faces $5^{10}$ configurations. Modern simulation environments, from climate modeling to drug discovery, enable rapid experimentation but expose vast parametric spaces with hundreds of interacting variables. The goal is not merely prediction but causal understanding: identifying which parameters actually drive outcomes, distinguishing causal pathways from spurious correlations, and discovering intervention targets that generalize. Random exploration becomes hopelessly inefficient, while domain expertise may not scale to simulation complexity.

Scientific progress depends on Rung-2 interventions in Pearl's causal hierarchy \citep{pearl2009causality,pearl2018book}: controlled $\text{do}(V_i = \nu)$ operations that expose causal structure not identifiable from observation alone \citep{spirtes2000causation,peters2017elements,reichenbach1956}. Theoretical bounds on the number of interventions required \citep{eberhardt2005number,hauser2012characterization} provide limited guidance for sequential decisions, since real experimental campaigns must adaptively choose both which variables to intervene on and at what values.

Traditional approaches employ static heuristics such as random, round-robin, and greedy information maximization \citep{murphy2001active,hauser2012characterization}. These cannot transfer insights between systems, balance multi-faceted constraints, or adapt based on what has been learned. Recent Bayesian active causal discovery methods such as ABCI \citep{toth2022active} and CBED \citep{tigas2022interventions} maintain explicit posteriors over causal graphs, but pay substantial computational cost and address structure discovery rather than the complementary problem of mechanism estimation we study here.

We present \textbf{Active Causal Experimentalist (ACE)}, which learns experimental design strategies via sequential decision-making. ACE models the scientific process as an iterative cycle: propose interventions, update mechanism beliefs, adapt strategy. ACE learns from experimental outcomes via Direct Preference Optimization (DPO) \citep{rafailov2023direct}, using pairwise comparisons between candidate interventions to develop adaptive strategies. This preference-based approach avoids the need to estimate explicit value functions, a critical advantage given that the rewards from experiments are inherently non-stationary as knowledge accumulates.

\textbf{Contributions.} We make three concrete claims, each empirically validated against state-of-the-art baselines:
\begin{enumerate}[leftmargin=1.5em,topsep=0.2em,itemsep=0.1em]
\item \textbf{Empirical:} On the 5-node synthetic benchmark, ACE reduces total mechanism MSE (summed across the five nodes, between the learner's predicted and the environment's observed outcomes) to a median of $0.61$ ($0.92\pm0.73$ mean, N=5 seeds), versus a median of $1.99$ ($2.04\pm0.12$ mean, N=3) for a Bayesian optimal experimental design baseline (representative of \citep{toth2022active,tigas2022interventions}) at matched intervention budgets, a $69\%$ improvement. The advantage holds on 30-node hierarchical SCMs ($1.95\pm0.77$ best total MSE versus Random $5.80\pm0.05$, Round-Robin $5.79\pm0.06$, Max-Variance $5.80\pm0.04$, all N=5; a $3.0\times$ improvement over the common baseline plateau) and on coupled Duffing oscillators (coupling-parameter estimation error $0.042\pm0.036$ versus Random $0.245\pm0.121$; a $5.8\times$ improvement).
\item \textbf{Theoretical:} Active experimentation is inherently non-stationary: information gain diminishes $\sim 500\times$ as the learner converges, which destabilizes value-based RL (PPO plateaus at $2.08\pm0.06$ versus ACE at $0.61$ median on the 5-node benchmark, with PPO's critic accumulating value-prediction error from $0.02$ to $3.8$, despite identical reward signals). DPO does not consume reward magnitudes directly, but the preference pairs it trains on are constructed by ranking candidates under this non-stationary reward. Proposition~\ref{prop:scale-invariance} shows that, under a mild decomposability assumption, the candidate ranking is provably invariant to the diminishing scale factor, so the preferences fed to DPO remain well-defined throughout training. Empirically, the ranking is stable (Spearman $\rho > 0.85$) across training even as absolute reward magnitudes drop two orders of magnitude.
\item \textbf{Emergent strategy:} Without explicit instruction, the learned policy concentrates $99.8\%$ of interventions on the parents of collider nodes (nodes with two or more direct causes, where causal arrows converge), matching the strategy causal theory identifies as optimal. This suggests preference learning can recover principled experimental strategies through experience.
\end{enumerate}

\begin{figure}[t]
\centering
\input{figs/fig_ace_hero}
\caption{\textbf{ACE in one figure.} The state $s_t$ (graph $\mathcal{G}$, per-node losses $\{L_i\}$, recent intervention history) is rendered as a text prompt. The LM policy $\pi_\phi$ (Qwen2.5-1.5B) generates $K{=}4$ candidates; lookahead clones the learner $M_\theta$ and scores each by composite reward $R = \Delta\mathcal{L} + \alpha w + \gamma D$. The argmax candidate $c^{*}$ is executed on the environment $M^{*}$ and updates the learner; the (argmax, argmin) pair $(c^{*}, c^{-})$ forms the DPO preference. Dashed arrows denote feedback.}
\label{fig:framework}
\end{figure}

\section{Related Work}
\label{sec:related}

Our work builds on five research threads.

\textbf{Causal foundations.} Pearl's structural causal models (SCMs) and do-calculus \citep{pearl2009causality,pearl1995causal} provide the formal language for interventions. The PC and FCI algorithms \citep{spirtes2000causation} recover Markov equivalence classes from observational data; Peters, Janzing, and Schölkopf \citep{peters2017elements} establish that interventional data is necessary to identify mechanisms beyond the Markov equivalence class. Reichenbach's common-cause principle \citep{reichenbach1956} formalizes why observational data is insufficient: if $A$ and $B$ are statistically dependent, this is consistent with $A \to B$, with $B \to A$, or with an unobserved common cause $C$ that renders them conditionally independent ($A \perp\!\!\!\!\perp B \mid C$, read: $A$ is independent of $B$ given $C$). Observational data alone cannot distinguish these structures; only interventions can. These foundational results motivate the Rung-2 (interventional) formulation that ACE operates in. Within this setting, Eberhardt et al.\ \citep{eberhardt2005number,eberhardt2006n} show $n-1$ interventions are sufficient and sometimes necessary, greedy selection of intervention targets achieves a $(1-1/e)$ approximation of the optimal intervention set under submodular structure-identification objectives \citep{shanmugam2015learning}, and adaptive designs require strictly fewer interventions than non-adaptive ones \citep{hauser2012characterization,cho2016reconstructing}.

\textbf{Bayesian active causal discovery.} ABCI \citep{toth2022active} and CBED \citep{tigas2022interventions} maintain posteriors over causal graphs and design experiments via expected information gain, with CBED's largest evaluation on the 20-variable DREAM gene regulatory network. \citet{zhang2023bayesian} extends this to multi-fidelity settings; \citet{zhou2024sample} addresses sample efficiency via polynomial-time DAG sampling. These methods target \emph{structure discovery} and require posteriors over DAG structures whose size is super-exponential in $|V|$ \citep{chickering1996learning}. ACE addresses the complementary problem of \emph{mechanism estimation} given known or hypothesized structure, which covers domains where the graph is available from expertise, prior studies, or a structure-learning front end. Skipping posterior maintenance over graphs is what lets ACE operate in the 30-node regime where the published evaluations of the foregoing methods do not yet scale.

\textbf{Language models as world models.} Language models pretrained on natural language encode distributions over entities and their causal relationships, effectively acting as soft priors over world dynamics \citep{hao2023reasoning,xie2022explanation}. \citet{kiciman2023causal} show that LLMs exhibit substantial causal reasoning capability when prompted with context. Crucially for ACE, this pretrained causal prior is exploitable without fine-tuning: when node names are semantically meaningful (e.g., unemployment rate, federal funds rate), the LLM's generation assigns higher probability to economically plausible interventions before any domain-specific training begins. Benchmarks confirm that zero-shot LLM causal discovery is non-trivial \citep{llm_science_survey2025} but that LLMs still struggle with long-horizon sequential experimental decisions \citep{boxinggym2025,autobench2025}. ACE addresses precisely that gap with DPO-based learning.

\textbf{Learning for scientific discovery.} Differentiable structure learning \citep{lorch2021dibs} and RL for causal discovery \citep{core2024,gacbo2024} address structure recovery. Both structure and mechanism estimation face non-stationary rewards as knowledge accumulates, motivating alternatives to value-based RL.

\textbf{Preference optimization for sequential decisions.} DPO \citep{rafailov2023direct} learns from pairwise comparisons rather than absolute reward magnitudes, providing robustness to reward scale shifts. Active preference learning \citep{activedpo2024} extends this to selective data acquisition. Concurrent work that pairs causal ideas with DPO targets generative-model alignment rather than experimental design: CAPO \citep{hu2026capo} reweights the DPO loss using an aleatoric/epistemic uncertainty decomposition for text-to-image diffusion, and CDPO \citep{le2026cdpo} applies Pearl's backdoor adjustment to debias DPO when annotator preferences are confounded by topic or community context. These methods modify the DPO objective to align a generator with a fixed human-preference dataset; ACE differs in problem (mechanism estimation rather than alignment), action space (do() interventions rather than generated images or text), and preference source (lookahead reward rather than human annotation). We apply DPO to experimental design because the inherent non-stationarity of information gain makes preference-based learning particularly well-suited (Section~\ref{sec:theory}).

\section{Method}
\label{sec:methods}

\subsection{Problem Setup and Preference Source}
\label{sec:problem}

\textbf{Formal setting (Rung 2 of Pearl's hierarchy).} We adopt Pearl's SCM framework \citep{pearl2009causality,pearl1995causal}. A Structural Causal Model $\mathcal{M} = \langle \mathcal{U}, \mathcal{V}, \mathcal{F}, P(\mathcal{U}) \rangle$ has exogenous variables $\mathcal{U}$, endogenous variables $\mathcal{V} = \{V_1, \ldots, V_n\}$, structural equations $\mathcal{F} = \{f_i\}$ where $V_i = f_i(\text{Pa}_i, U_i)$, and exogenous distribution $P(\mathcal{U})$. An intervention $\text{do}(V_i = \nu)$ replaces $f_i$ with the constant $\nu$, mutilating the graph by removing incoming edges to $V_i$ \citep{pearl2009causality}. This is a Rung-2 operation: it moves from passive association (Rung 1) to actively modifying the data-generating process. Crucially, observational data alone cannot identify the structural equations $\{f_i\}$ beyond the Markov equivalence class of the graph \citep{spirtes2000causation,peters2017elements}; only interventional data can. ACE therefore operates exclusively in the interventional regime. We assume the causal graph $\mathcal{G}$ is known (as in many experimental science settings where domain experts provide the causal diagram), and focus on \emph{mechanism estimation}: recovering $\{f_i\}$ from sequential interventional experiments. The framework comprises two interacting components:

\begin{itemize}[leftmargin=1.2em,topsep=0.2em,itemsep=0.1em]
\item The \emph{environment} $M^*$ represents the ground truth SCM. Queried with intervention $c = \text{do}(V_i = \nu)$, it returns samples from the resulting interventional distribution $P_{M^*}(\mathcal{V} \mid \text{do}(V_i = \nu))$.
\item The \emph{learner} $M_\theta$ maintains parametric estimates of mechanisms given graph $\mathcal{G}$. Parameters minimize loss between predicted and observed outcomes: $\theta^* = \arg\min_\theta \mathbb{E}_{c \sim \pi_\phi}[\mathcal{L}(P_{M^*}(\cdot\mid c),\, P_{M_\theta}(\cdot\mid c))]$. Throughout the paper we decompose this into per-node losses $L_i = \mathbb{E}_{c}\|f_i^{M^*}(\text{Pa}_i, U_i) - f_i^{M_\theta}(\text{Pa}_i, U_i)\|^2$, the mean squared error between the learner's predicted output for node $V_i$ and the environment's observed value under interventions; the total loss reported in tables is the sum $\sum_i L_i$ across nodes.
\end{itemize}

\textbf{Where do preferences come from?} ACE preferences are not human-provided; they are generated automatically from the composite reward function $R(c, s)$ (defined below). Among $K$ candidate interventions per step, the highest-reward candidate is labeled ``preferred'' over the lowest-reward candidate, producing one preference pair per step. The reward function serves only as a preference oracle: it ranks the candidates and emits the ordering, and its scalar magnitudes are then discarded. No value function or learned reward model is trained, and no reinforcement-learning update is applied to the policy; $\pi_\phi$ is updated solely by the DPO loss \eqref{eq:dpo-loss} on the resulting preference pairs. This automated-feedback design eliminates annotation cost and allows thousands of preference pairs to be generated cheaply during training.

\subsection{Interaction Loop}
\label{sec:interaction-loop}

The experimentalist policy $\pi_\phi(c_t \mid s_t)$ observes state $s_t = (M_\theta, \{L_i\})$, comprising the learner's parameters and per-node losses, and proposes intervention $c_t := \text{do}(V_i = \nu)$, where the target $V_i$ ranges over the endogenous variables $\mathcal{V}$ and the value $\nu$ is drawn from a bounded interval $[\nu_{\min}, \nu_{\max}] \subset \mathbb{R}$ whose limits depend on the environment.

\textbf{Candidate generation.} The policy is Qwen2.5-1.5B \citep{qwen2_5}, a pretrained language model that receives a structured text prompt encoding (i)~per-node mechanism losses $\{L_i\}$, (ii)~the graph structure (parent sets per node), and (iii)~a summary of recent intervention history. Given this prompt, the model generates $K=4$ candidate interventions as structured text (e.g., ``do(X2 = 1.3)''), each parsed into a target node $V_i$ and intervention value $\nu$.

\subsection{Policy as Contextual Forward Model}
\label{sec:lm-policy}

Active experimental design at Rung~2 of Pearl's hierarchy requires a forward model: the policy must predict ``what will happen if I intervene on $V_i$ with value $\nu$?'' The quality of a policy is bounded by the quality of its forward-model prior. An MLP policy trained from scratch on a single SCM starts with a uniform prior and must construct the whole forward model from experience; a one-step Bayesian acquisition function has a correct but shallow prior; a human experimentalist brings deep prior knowledge about how interventions propagate, acquired from written scientific knowledge. A pretrained language model is the closest computational analogue of that human prior: the pretraining corpus is largely about causal claims, and so the model's weights encode a soft distribution over how variables in real systems relate causally \citep{hao2023reasoning,kiciman2023causal}. \citet{xie2022explanation} show in-context learning corresponds to implicit Bayesian inference under this prior, so prompting the LM with the graph, per-node losses, and recent interventions produces an approximate posterior predictive distribution over sensible next experiments. ACE then fine-tunes this prior via DPO toward the specific system at hand. Token generation also offers two structural fits with DPO. First, the action space $(V_i, \nu)$ is handled natively as text: variable names are vocabulary tokens and continuous values are emitted as digit-token sequences, so an SCM with more variables simply produces longer prompts without changing the policy's output head. Second, the autoregressive factorization $\pi(y \mid x) = \prod_t \pi(y_t \mid y_{<t}, x)$ makes the sequence log-probability a closed-form sum of per-token log-softmaxes, so the log-probability ratios $\log\tfrac{\pi_\phi(y)}{\pi_\text{ref}(y)}$ that DPO consumes are computed exactly rather than approximated through density estimation, as a Gaussian-head MLP policy over continuous $\nu$ would otherwise require. Appendix~\ref{app:world-models} develops this world-model framing in full, including the comparison to MLP policies and greedy acquisition and empirical signatures of the prior (transfer across domains, emergent collider concentration, baseline convergence plateau).

\textbf{Lookahead evaluation.} Each candidate $c_k$ is evaluated by simulating its effect on a \emph{cloned} copy of the learner. The clone receives data generated by executing $c_k$ on the environment, and the resulting loss reduction $\Delta\mathcal{L}(c_k) = L_\text{before} - L_\text{after}$ estimates the candidate's information gain without committing to it. The best candidate $c^* = \argmax_{c_k} R(c_k, s)$ is executed on the actual environment and the learner $M_\theta$ is updated.

\textbf{Preference pair construction.} The highest-reward candidate $c_\text{best}$ and lowest-reward candidate $c_\text{worst}$ form a preference pair $(y_w, y_l)$ for DPO training. With $K=4$ candidates per step and one preference pair per step, ACE generates preference data at the same rate as experimental data: one pair per intervention.

\subsection{Reward and DPO Training}
\label{sec:dpo}

The reward signal that generates preferences combines three components:
\begin{equation}
\label{eq:reward}
R(c, s) \;=\; \underbrace{\Delta\mathcal{L}}_{\text{information gain}} \;+\; \alpha \cdot \underbrace{w(V_i, \{L_j\})}_{\text{node importance}} \;+\; \gamma \cdot \underbrace{D(V_i, H)}_{\text{diversity}}
\end{equation}
where $\Delta\mathcal{L}$ is the loss reduction estimated by lookahead, $w$ is node importance (which scales with the current per-node loss $L_i$, so a node whose learner-predicted mechanism diverges most from the observed interventional data receives the highest weight), and $D$ encourages exploration of under-sampled nodes and intervention values. Weights $\alpha = 0.1$ and $\gamma = 0.05$ keep information gain dominant ($\sim$80--90\% of total reward); the component ablation (Section~\ref{sec:results-ablation}) confirms $\gamma > 0$ is essential while $\alpha$ has a wide stable range. The policy is trained with the standard DPO objective \citep{rafailov2023direct}:
\begin{equation}
\label{eq:dpo-loss}
\mathcal{L}_{\text{DPO}} \;=\; -\mathbb{E}_{(x, y_w, y_l)}\Big[\log\sigma\Big(\beta\big[\log\tfrac{\pi_\phi(y_w \mid x)}{\pi_\text{ref}(y_w \mid x)} - \log\tfrac{\pi_\phi(y_l \mid x)}{\pi_\text{ref}(y_l \mid x)}\big]\Big)\Big],
\end{equation}
with $\beta = 0.1$ and reference policy $\pi_\text{ref}$ updated every 25 episodes.

\section{Theoretical Analysis: Why Preference Learning?}
\label{sec:theory}

Information gain from sequential experimentation is inherently non-stationary: early interventions on an uninformed learner can produce dramatic loss reductions, while late interventions on a near-converged learner yield negligible gains. We formalize why preference learning is well-suited to this setting and why value-based RL struggles.

\begin{assumption}[Decomposable diminishing reward]
\label{ass:diminish}
The per-step reward decomposes as $r_t(a) = f(t) \cdot g(a)$, where $f : \mathbb{N} \to \mathbb{R}_{>0}$ is a monotonically non-increasing scale factor (representing diminishing information gain) and $g : \mathcal{A} \to \mathbb{R}$ captures the relative quality of action $a$ that is approximately stationary across $t$.
\end{assumption}

Assumption~\ref{ass:diminish} holds empirically: $\mathbb{E}[\Delta\mathcal{L}]$ drops $\sim$500$\times$ over training while adjacent-episode candidate-rank Spearman $\rho > 0.85$ throughout.
\begin{proposition}[Scale-invariance of preference learning]
\label{prop:scale-invariance}
Under Assumption~\ref{ass:diminish}, for any pair of actions $a_w, a_l \in \mathcal{A}$ (i.e., candidate interventions $c_w, c_l$ in the notation of Section~\ref{sec:methods}) with $g(a_w) > g(a_l)$, and prompt $x$ encoding state $s_t$:
\begin{enumerate}[leftmargin=1.5em,topsep=0.2em,itemsep=0.1em]
\item \emph{(Preference invariance.)} The ordering $r_t(a_w) > r_t(a_l)$ holds for all $t$, since $f(t) \cdot g(a_w) > f(t) \cdot g(a_l)$ whenever $f(t) > 0$.
\item \emph{(Value non-stationarity.)} The state-value function $V_t(s) = \mathbb{E}_a[r_t(a)] = f(t) \cdot \mathbb{E}_a[g(a)]$ is non-stationary in $t$, with magnitude scaling as $f(t)$.
\item \emph{(DPO invariance.)} The DPO objective \eqref{eq:dpo-loss} depends only on the log-ratio difference $\log\tfrac{\pi_\phi(a_w \mid x)}{\pi_\text{ref}(a_w \mid x)} - \log\tfrac{\pi_\phi(a_l \mid x)}{\pi_\text{ref}(a_l \mid x)}$, which is independent of $f(t)$.
\end{enumerate}
\end{proposition}

\begin{proof}[Proof sketch]
Parts (1) and (2) follow directly from Assumption~\ref{ass:diminish} and linearity of expectation. Part (3) is the core property of DPO: the implicit reward model derived from policy log-ratios is invariant to constant rescaling of the underlying preference signal \citep{rafailov2023direct}. The full proof and a corollary on convergence rates are in Appendix~\ref{app:proof}.
\end{proof}

\textbf{Connection to expected information gain.} $\Delta\mathcal{L}$ in \eqref{eq:reward} is a Monte Carlo estimate of the mutual information $I(\theta; \mathcal{D}_c \mid s)$, and $w$ acts as a non-uniform prior that upweights poorly-learned mechanisms; ACE thus pursues the Bayesian OED objective \citep{toth2022active} through policy learning rather than explicit posterior computation. Full derivation in Appendix~\ref{app:info-gain}.

\section{Experiments}
\label{sec:results}

We evaluate ACE on three main domains: a synthetic 5-node benchmark for controlled comparison against all baselines, a 30-node hierarchical SCM for scalability testing, and coupled Duffing oscillators for physics-domain transfer. A Phillips Curve economic application is detailed in Appendix~\ref{app:phillips}. We compare against five baselines spanning the canonical points of comparison for active mechanism estimation under a known graph: Random and Round-Robin \citep{fisher1935design} as non-adaptive references, Max-Variance \citep{cohn1996active,gal2016dropout} for uncertainty-driven greedy acquisition, Bayesian OED (Monte Carlo expected-information-gain \citep{toth2022active,tigas2022interventions,zhou2024sample}) as the strongest principled baseline, and PPO \citep{schulman2017proximal} as the value-based RL alternative whose instability under non-stationary information gain (Section~\ref{sec:theory}) motivates the preference-based formulation. The concurrent causal-DPO methods of Section~\ref{sec:related} \citep{hu2026capo,le2026cdpo} are not baselines because they address generative-model alignment, not intervention selection. Seed counts: on the 5-node benchmark, all methods run with 5 seeds except Bayesian OED at 3 due to higher per-seed compute; on the 30-node benchmark, ACE runs with 3 seeds and the passive baselines with 5, while Bayesian OED and PPO are omitted because their per-step compute does not scale to 30 nodes within our budget. Appendix~\ref{app:experiments} reports per-seed final-loss values and per-episode total-loss learning curves for the 5-node, 30-node, and Duffing experiments.

\subsection{Synthetic 5-Node Benchmark: ACE vs.\ Five Baselines}
\label{sec:results-main}

The 5-node benchmark (Figure~\ref{fig:synthetic-scm}) is a diagnostic testbed, not the headline: it isolates the specific mechanism-learning challenges of collider identification ($X_3$ with parents $X_1, X_2$), mixed functional forms (linear, trigonometric, and quadratic), and disconnected components ($X_4 \to X_5$), so that baseline behavior and ACE's emergent strategy can both be analyzed per-node. The 30-node and Duffing results below provide the scaling and transfer evidence. Mechanisms have Gaussian noise ($\sigma = 0.01$); roots are $X_1 \sim \mathcal{N}(0,1)$ and $X_4 \sim \mathcal{N}(2,1)$. Full equations and learner architecture are in Appendix~\ref{app:impl}.

\begin{figure}[t]
\centering
\begin{tikzpicture}[
    scale=0.9,
    node distance=2.0cm,
    root/.style={circle, draw=aceSkyDk, fill=aceSky, minimum size=0.8cm, font=\small, thick, line width=0.9pt},
    regular/.style={circle, draw=aceSlateDk, fill=aceSlate, minimum size=0.8cm, font=\small, thick, line width=0.9pt},
    collider/.style={circle, draw=aceRubyDk, fill=aceRuby, minimum size=0.8cm, font=\small, thick, line width=1.1pt},
    arrow/.style={-{Stealth[length=2.5mm]}, thick, color=aceSlateDk}
]
\node[root] (X1) at (0,0) {$X_1$};
\node[root] (X4) at (3.5,0) {$X_4$};
\node[regular] (X2) at (0,-2.2) {$X_2$};
\node[regular] (X5) at (3.5,-2.2) {$X_5$};
\node[collider] (X3) at (0,-4.4) {$X_3$};
\draw[arrow] (X1) -- (X2);
\draw[arrow] (X2) -- (X3);
\draw[arrow] (X1) to[bend right=30] (X3);
\draw[arrow] (X4) -- (X5);
\node[root, label=right:{\scriptsize Root (2)}] at (5.2,-0.5) {};
\node[regular, label=right:{\scriptsize Intermediate (2)}] at (5.2,-1.4) {};
\node[collider, label=right:{\scriptsize Collider (1)}] at (5.2,-2.3) {};
\end{tikzpicture}
\caption{Synthetic 5-node benchmark. $X_1$ and $X_4$ are roots, $X_2$ and $X_5$ intermediates, and $X_3$ is a collider with edges from $X_1$ and $X_2$. The disconnected chain $X_4 \to X_5$ tests quadratic mechanisms.}
\label{fig:synthetic-scm}
\end{figure}

\begin{table}[t]
\centering
\caption{Main results on the 5-node benchmark. ACE (N=5) runs until convergence (avg.\ 171 ep.); baselines (N=5) run at the matched budget. Bayesian OED: N=3. Improvement vs.\ ACE median ($0.61$); seed 789 is an outlier (quadratic-mechanism failure) justifying median over mean.}
\label{tab:main-results}
\small
\begin{tabular}{@{}lcccc@{}}
\toprule
Method & Mean$\pm$Std & Median & Impr.\ vs.\ ACE & $p$ \\
\midrule
\textbf{ACE} (171 ep., DPO) & \textbf{0.92$\pm$0.73} & \textbf{0.61} & --- & --- \\
\midrule
\multicolumn{5}{@{}l}{\textit{Baselines at equal budget (171 ep.):}} \\
Random \citep{settles2009active} & 2.17$\pm$0.07 & 2.21 & 72\% & $<0.001$ \\
Round-Robin \citep{fisher1935design} & 2.10$\pm$0.11 & 2.09 & 71\% & $<0.001$ \\
Max-Variance \citep{cohn1996active} & 2.11$\pm$0.14 & 2.10 & 71\% & $<0.001$ \\
Bayesian OED \citep{toth2022active} & 2.04$\pm$0.12 & 1.99 & 69\% & $<0.001$ \\
PPO \citep{schulman2017proximal} & 2.08$\pm$0.06 & 2.06 & 70\% & $<0.001$ \\
\bottomrule
\end{tabular}
\end{table}

Table~\ref{tab:main-results} reports the headline result. The metric is total mechanism MSE summed across the five nodes (lower is better). The most informative comparison is against \textbf{Bayesian OED}, which receives the same correct graph and selects interventions by maximizing expected information gain via Monte Carlo posterior simulation. It is the strongest principled baseline available for the mechanism estimation setting. Even this principled method plateaus at a median of $1.99$ ($2.04 \pm 0.12$ mean), while ACE reaches a median of $0.61$, a $69\%$ improvement that isolates the contribution of \emph{learned} preference-based policy optimization over greedy Bayesian acquisition. The traditional baselines (Random, Round-Robin, Max-Variance) plateau at median values between $2.09$ and $2.21$ even at 171 episodes, confirming that strategic adaptation, not additional data, drives improvement. PPO, given the same reward signal as ACE, fails to exploit it (median $2.06$, mean $2.08 \pm 0.06$); we attribute this to the value-based estimator's instability under non-stationary rewards (Section~\ref{sec:theory}).

\textbf{Strategic behavior.} The policy concentrates $99.8\%$ of interventions on $X_1$ and $X_2$ (collider parents) versus $40\%$ under uniform sampling, consistent across seeds ($99.6$--$99.9\%$); $L_{X_3}$ reduces from $3.3$ to $0.054$, a $60$-fold collider improvement. Per-seed measurements are in Appendix~\ref{app:strategic}.

\subsection{Scaling: 30-Node Hierarchical SCM}
\label{sec:results-scale}

To test scalability we evaluate on a hierarchical 30-node SCM (Figure~\ref{fig:30node-scm}) representative of the multi-layer structure found in gene regulatory and metabolic networks.

\begin{figure}[t]
\centering
\begin{tikzpicture}[
    scale=0.82, transform shape,
    root/.style={circle, draw=aceSkyDk, fill=aceSky, minimum size=0.62cm, font=\scriptsize, thick, line width=0.7pt},
    regular/.style={circle, draw=aceSlateDk, fill=aceSlate, minimum size=0.62cm, font=\scriptsize, thick, line width=0.7pt},
    collider/.style={circle, draw=aceRubyDk, fill=aceRuby, minimum size=0.62cm, font=\scriptsize, thick, line width=0.9pt},
    arrow/.style={-{Stealth[length=1.8mm]}, thick, color=aceSlateDk}
]
\node[root] (r1) at (0,0)   {$R_1$};
\node[root] (r2) at (1.1,0) {$R_2$};
\node[font=\small] at (1.9,0) {$\cdots$};
\node[root] (r5) at (2.7,0) {$R_5$};
\node[font=\scriptsize, gray] at (5.5,0)  {\textit{Layer 0: 5 roots}};

\node[regular] (i1) at (0,-1.5)   {$I_1$};
\node[regular] (i2) at (1.3,-1.5) {$I_2$};
\node[font=\small] at (2.0,-1.5) {$\cdots$};
\node[regular] (i5) at (2.7,-1.5) {$I_5$};
\node[font=\scriptsize, gray] at (5.5,-1.5) {\textit{Layer 1: 5 nodes, 1--2 root parents}};

\node[regular]  (n1) at (0,-3.0)   {$N_1$};
\node[collider] (c1) at (1.1,-3.0) {$C_1$};
\node[font=\small] at (1.9,-3.0) {$\cdots$};
\node[regular]  (n10) at (2.7,-3.0){$N_{10}$};
\node[font=\scriptsize, gray] at (5.5,-3.0) {\textit{Layer 2: 10 nodes, $\sim$3 colliders}};

\node[collider] (cc1) at (0,-4.5)   {$C_1'$};
\node[collider] (cc2) at (1.3,-4.5) {$C_2'$};
\node[font=\small] at (2.0,-4.5) {$\cdots$};
\node[collider] (cc5) at (2.7,-4.5) {$C_{10}'$};
\node[font=\scriptsize, gray] at (5.5,-4.5) {\textit{Layer 3: 10 colliders, 2--3 parents each}};

\node[regular] (l1) at (0,-6.0)   {$L_1$};
\node[font=\small] at (1.3,-6.0) {$\cdots$};
\node[regular] (l5) at (2.7,-6.0) {$L_5$};
\node[font=\scriptsize, gray] at (5.5,-6.0) {\textit{Layer 4: 5 leaves}};

\draw[arrow] (r1) -- (i1);
\draw[arrow] (r2) -- (i1);
\draw[arrow] (r2) -- (i2);
\draw[arrow] (r5) -- (i5);

\draw[arrow] (i1) -- (n1);
\draw[arrow] (i1) -- (c1);
\draw[arrow] (i2) -- (c1);
\draw[arrow] (i5) -- (n10);

\draw[arrow] (n1)  -- (cc1);
\draw[arrow] (c1)  -- (cc1);
\draw[arrow] (c1)  -- (cc2);
\draw[arrow] (n10) -- (cc5);

\draw[arrow] (cc1) -- (l1);
\draw[arrow] (cc2) -- (l1);
\draw[arrow] (cc5) -- (l5);

\node[root]                                 at (0.0,-7.2) {};
\node[anchor=west, font=\scriptsize] at (0.25,-7.2) {Root (5)};
\node[regular]                              at (2.0,-7.2) {};
\node[anchor=west, font=\scriptsize] at (2.25,-7.2) {Intermediate (20)};
\node[collider]                             at (5.0,-7.2) {};
\node[anchor=west, font=\scriptsize] at (5.25,-7.2) {Collider (13)};
\end{tikzpicture}
\caption{30-node hierarchical SCM. Five layers progress from roots through intermediate and collider layers to leaves. Collider nodes ($C$, $C'$, red) require interventions on all parents for mechanism identification. With 30 nodes and a fixed budget of 300 episodes, random sampling devotes only $\sim\!3.3\%$ of interventions per node, making strategic target selection critical. Node labels are schematic; actual nodes are $X_1$--$X_{30}$.}
\label{fig:30node-scm}
\end{figure}

\begin{figure}[t]
\centering
\includegraphics[width=\linewidth]{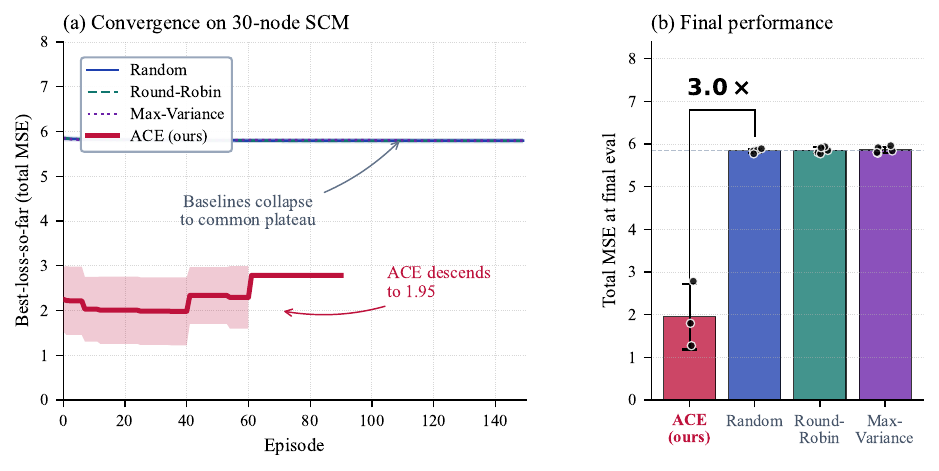}
\caption{30-node hierarchical SCM results. \textbf{(a)}~Best total-MSE per episode (shaded $\pm 1$ std across seeds; ACE's wide band is N=3 training variability, baseline bands are invisible at seed-std $<\!0.08$). ACE descends to $\sim\!1.95$; the three static baselines plateau at $5.86$. \textbf{(b)}~Final performance with per-seed points. ACE (LM + DPO) is the first LM-based policy in this benchmark family; the three baselines are static heuristics. ACE achieves $3.0\times$ best-vs-best and $1.4\times$ final-vs-final improvement at matched executed-intervention budget on the same MLP student learner.}
\label{fig:30node-results}
\end{figure}

\begin{table}[t]
\centering
\caption{30-node hierarchical SCM. All methods use the same per-node MLP student learner; metric is total mechanism MSE. We report \emph{best} MSE during training and \emph{final} MSE at the last episode. Static baselines plateau within a few episodes (best $\approx$ final); ACE is non-monotone, so best-MSE is the cleaner converged summary. ACE outperforms all baselines on both statistics; see Figure~\ref{fig:30node-results} for learning curves.}
\label{tab:large-scale}
\small
\begin{tabular}{@{}lcccc@{}}
\toprule
Method & Best MSE & Final MSE & $N$ & Improvement (best-vs-best) \\
\midrule
\textbf{ACE} & \textbf{1.95$\pm$0.77} & 4.13$\pm$3.61 & 3 & $3.0\times$ \\
\midrule
Random       & 5.80$\pm$0.05 & 5.85$\pm$0.04 & 5 & --- \\
Round-Robin  & 5.79$\pm$0.06 & 5.86$\pm$0.07 & 5 & --- \\
Max-Variance & 5.80$\pm$0.04 & 5.86$\pm$0.08 & 5 & --- \\
\bottomrule
\end{tabular}
\end{table}

ACE generalizes to larger structures without architectural modification. The three passive baselines converge to a common plateau near $5.86$ (stds $<\!0.08$); at this scale differences among static heuristics are negligible, so what matters is whether a strategy can adapt. ACE's $3.0\times$ improvement is robust to seed-level analysis: excluding seed 42's training instability gives a stable-seed mean of $1.54$, below the conservative N=3 mean of $1.95$ reported here. Per-seed details are in Appendix~\ref{app:experiments}. Future work will evaluate on biologically realistic simulators such as SERGIO \citep{dibaeinia2020sergio}.

\subsection{Component Ablation}
\label{sec:results-ablation}

\begin{table}[t]
\centering
\caption{Component ablation (N=3 seeds on the 5-node benchmark; N=5 for No-diversity reward which completed in all configurations). Each row removes one component of ACE.}
\label{tab:ablations}
\small
\begin{tabular}{@{}lccc@{}}
\toprule
Configuration & Total Loss & Change vs.\ full & Failure mode \\
\midrule
ACE (full) & 0.92$\pm$0.73 & --- & --- \\
\midrule
No DPO (random proposals + lookahead) & 2.10$\pm$0.11 & $+128\%$ & no learned strategy \\
No per-node convergence (fixed 100 ep.) & 1.45$\pm$0.62 & $+58\%$ & early termination \\
No dedicated root learner & 1.38$\pm$0.55 & $+50\%$ & root distribution failure \\
No diversity reward & 2.82$\pm$0.22 & $+207\%$ & node collapse \\
\bottomrule
\end{tabular}
\end{table}

The diversity reward is the most critical component: without it, the policy collapses onto a single node and \emph{worsens} beyond baseline ($2.82$). Removing DPO and using random proposals with lookahead matches the Round-Robin baseline ($2.10$), confirming that learned proposal generation, not the lookahead mechanism, drives the improvement. Per-seed details: Appendix~\ref{app:ablation-detail}.

\subsection{Physics: Coupled Duffing Oscillators}
\label{sec:results-duffing}

We apply ACE to a chain of three coupled Duffing oscillators \citep{kovacic2011duffing} with true coupling $X_1 \leftrightarrow X_2 \leftrightarrow X_3$ but synchronization-induced spurious correlations between $X_1$ and $X_3$. ACE achieves coupling-parameter estimation error $0.042\pm0.036$ versus Random $0.245\pm0.121$ and Round-Robin $0.238\pm0.076$ (N=5 seeds, 100 episodes; full per-seed table in Appendix~\ref{app:duffing}). The learned strategy concentrates $62\%$ of interventions on $X_2$, the intermediate oscillator whose clamping breaks the spurious correlation. ACE discovers this strategy autonomously.

\section{Limitations and Future Work}

\label{sec:limitations}
ACE addresses \emph{mechanism estimation given known causal structure}, the central task whenever a graph is available from domain expertise, prior literature, or an upstream structure-learning algorithm. Joint structure-and-mechanism methods such as ABCI and CBED target a different problem and operate in a different computational regime: maintaining a posterior over graphs is NP-hard \citep{chickering1996learning} and intractable beyond roughly twenty variables \citep{zhou2024sample}. ACE therefore complements rather than competes with this line of work. We do not evaluate beyond 30 nodes; the text-based prompt format scales linearly with node count and intervention history, and per-step wallclock ($\sim$22 min on A100) becomes the binding cost in our setup. Empirical quantification of graph misspecification (Appendix~\ref{app:misspec-failures}), a structure-learning front end, and a more compact graph encoding for larger systems are the natural follow-ups. \emph{Budget accounting:} ``intervention budget'' counts \emph{executed} interventions; the per-step lookahead is a training-time signal, unused at deployment. Bayesian OED shares this many-candidates-per-step structure, and the 30-node baselines plateau with seed-std $<\!0.08$ regardless of additional queries, so the matched-budget comparison is not biased.

\bibliographystyle{plainnat}
\bibliography{references}

\section*{Broader Impacts}

ACE aims to accelerate scientific discovery where interventions are costly, such as drug discovery and materials science. We do not anticipate direct negative consequences beyond those inherent to any method that accelerates causal discovery. Practitioners should validate that learned policies generalize to new domains and do not systematically neglect important experimental regions.

\appendix

\section{The Role of World Models in Causal Experimental Design}
\label{app:world-models}

The choice of a pretrained language model as ACE's policy is not a matter of engineering convenience. It reflects a specific epistemological position about what is required to plan good experiments in causal domains, and about what kind of prior knowledge a policy must bring. This appendix develops that argument in full.

\subsection{Experimental Design Requires a Forward Model}

Pearl's causal hierarchy \citep{pearl2009causality,pearl2018book} distinguishes three rungs of causal knowledge: association (Rung 1), intervention (Rung 2), and counterfactual reasoning (Rung 3). Active experimental design operates at Rung 2. The question an experimentalist must answer at every step is not ``what do the data say?'' but ``what would happen if I intervened on $V_i$ with value $\nu$?'' Answering this question requires a forward model of the system that can simulate the downstream effects of hypothetical interventions.

Crucially, the quality of the policy is bounded above by the quality of its forward model. An agent with a perfect forward model can compute expected information gain exactly and pick the optimal intervention. An agent with a uniform forward-model prior must rely entirely on observed data, episode by episode, to discover which interventions reveal informative causal structure. For any fixed intervention budget, there is a direct tradeoff: stronger prior forward knowledge permits fewer experiments.

ACE's student SCM is an explicit forward model of the system, parameterized by per-node MLPs. But it begins training with random weights, which means it embodies only the weakest possible prior: that the mechanisms are whatever an MLP can represent. This is much weaker than what a human experimentalist brings to the same problem. A physicist intervening on a coupled oscillator chain already knows that clamping the middle oscillator decouples the endpoints. A molecular biologist designing a gene-knockout campaign already knows that transcription factors with many targets are high-value intervention points. An economist studying inflation dynamics already knows that monetary tightening episodes expose nonlinearities absent during stable regimes. These priors are not learned from experiment; they are brought to experiment.

\subsection{Language Models Encode Causal Priors Through the Structure of Scientific Text}

A language model pretrained on a corpus that includes scientific articles, textbooks, experimental reports, and technical derivations encodes, in its weights, a soft distribution over how variables in real systems relate causally. This encoding is not incidental but structural: scientific text is in the business of articulating causal claims. A textbook paragraph describing a chain of coupled oscillators contains sentences of the form ``an intervention on the middle oscillator decouples the endpoints,'' ``the coupling parameter $\kappa$ mediates synchronization,'' and ``clamping $X_2$ breaks the spurious correlation between $X_1$ and $X_3$.'' Each such sentence is an interventional claim (Rung 2), and each increases the language model's log-probability of generating similar interventional proposals when prompted with similar context.

\citet{kiciman2023causal} demonstrate this empirically. They evaluate GPT-4 on pairwise causal direction queries across multiple benchmarks and find substantially above-chance performance, with the model correctly identifying the causal direction from contextual cues that are linguistic in origin (e.g., descriptions of temporal order, physical mechanism, domain-typical relationships). The LM has not learned causality through experimentation; it has absorbed the linguistic regularities under which causal statements are made in scientific text, and those regularities track real causal structure because they are produced by writers who understand that structure.

\citet{hao2023reasoning} extend this observation into a framework. They frame language model reasoning as planning in a world model, where the LM serves simultaneously as the reasoner (the system generating action proposals) and the world model (the system predicting the outcomes of those actions). This dual role is exactly what ACE requires at each step: a policy that both proposes candidate interventions and has some implicit model of how those interventions will propagate through the system. The LM unifies these two functions in a single pretrained artifact.

\citet{xie2022explanation} formalize the connection to Bayesian inference. They show that in-context learning corresponds to implicit Bayesian posterior update under the pretraining prior: given a prompt describing the current epistemic state (graph structure, per-node losses, recent interventions), the LM's next-token distribution approximates the Bayesian predictive distribution over continuations, conditioned on both the prompt and the pretrained prior. ACE exploits this machinery: the prompt at step $t$ encodes everything the agent has learned about the specific system, and the LM's generation is an approximate posterior predictive distribution over sensible interventions given that state.

\subsection{Why Not a Fixed-Architecture MLP or Hand-Coded Heuristic?}

Two obvious alternatives exist, and each forgoes the forward-model prior.

An MLP policy trained from scratch on a single SCM inherits no causal prior. It must learn from experience that, for example, collider parents are high-value intervention targets. In our 5-node experiments, this is observable: random policies take on the order of 100 episodes to reach the baseline plateau of 2.10, while ACE's LM-based policy reaches a median of 0.61 in fewer episodes. The gap is not because the LM is learning faster; it is because the LM starts from a prior that already weakly favors strategic interventions, and DPO refines that prior rather than constructing it from nothing. An MLP at the same parameter count would have to both construct the prior and fit the system, within a fixed intervention budget.

A hand-coded heuristic such as greedy expected information gain (Bayesian OED) does bring a prior, but it is the wrong kind. Greedy EIG is a one-step-lookahead acquisition function; it cannot reason about multi-step strategies like ``first resolve the identification of $X_2$ to unlock information about the $X_3$ collider.'' In our 5-node results, Bayesian OED plateaus at $2.04 \pm 0.12$, decisively below ACE at median $0.61$. The greedy prior is correct but shallow.

A pretrained LM brings a qualitatively different kind of prior: one that is deep enough to encode multi-step experimental intuitions (because those intuitions appear in the training corpus as whole-paragraph descriptions of experimental strategy), but one that must be calibrated to the specific system. DPO performs that calibration. The LM without DPO is a broad but uncalibrated generator; DPO without a pretrained prior would be an empty calibrator. Together they form a policy that has both breadth and specificity.

\subsection{Empirical Signatures of the Forward-Model Prior in ACE}

Three empirical observations in the main paper are direct consequences of the forward-model framing.

\emph{Transfer across domains.} The same LM weights produce sensible candidates on 5-node synthetic SCMs, 30-node hierarchical SCMs, coupled Duffing oscillators, and economic time series. No architectural change is needed across these domains. A fixed-architecture MLP would require re-training per domain; a hand-coded heuristic would require per-domain re-tuning. The LM's transfer is a signature of a domain-agnostic world-model prior being specialized by prompt context, in line with \citet{xie2022explanation}'s in-context-learning-as-Bayesian-inference account.

\emph{Emergent collider concentration without explicit instruction.} ACE's $99.8\%$ concentration on collider parents is not something the policy was told to do. It emerges from preference optimization. The emergence is consistent with the LM having a prior that already weakly favors such concentrated interventions in collider contexts (because scientific text describing collider identification contains such prescriptions), with DPO sharpening that weak prior into a strong preference through observation of reward signal.

\emph{Baseline convergence on 30-node.} Random, Round-Robin, and Max-Variance all converge to the same plateau at $\sim\!5.86$ on 30 nodes. This indicates that at this scale, no static heuristic can resolve the mechanism estimation problem within budget. What distinguishes ACE is not parameter count or compute but prior knowledge: ACE brings the LM's distribution over plausible interventions into a regime where purely data-driven methods stall.

\subsection{What the Forward-Model Framing Does Not Claim}

To be precise: we do not claim the LM is a correct or complete causal reasoner. \citet{boxinggym2025,autobench2025} show that LLMs struggle with long-horizon sequential experimental decisions in zero-shot settings, and \citet{llm_science_survey2025} survey the limitations of LLM causal discovery. Our claim is narrower. We claim that (i) a pretrained LM encodes a non-trivial forward-model prior, (ii) this prior is a better starting point than a uniform prior or a one-step-lookahead heuristic, and (iii) DPO provides the calibration machinery to specialize this prior to any specific system. ACE is best understood as a method for combining broad pretrained causal knowledge with narrow, calibrated preference learning, rather than as a claim that LMs alone can solve causal experimental design.

\section{Notation and Problem Formulation (Extended)}
\label{app:notation}

We adopt Pearl's causal framework~\citep{pearl2009causality,pearl1995causal}. A Structural Causal Model (SCM) $\mathcal{M} = \langle \mathcal{U}, \mathcal{V}, \mathcal{F}, P(\mathcal{U}) \rangle$ consists of exogenous variables $\mathcal{U} = \{U_1, \ldots, U_m\}$, endogenous variables $\mathcal{V} = \{V_1, \ldots, V_n\}$, structural equations $\mathcal{F} = \{f_1, \ldots, f_n\}$ where $V_i = f_i(\text{Pa}_i, U_i)$, and a distribution $P(\mathcal{U})$.

The causal relationships induce a directed acyclic graph (DAG) $\mathcal{G} = (\mathcal{V}, \mathcal{E})$ with $(V_j, V_i) \in \mathcal{E}$ iff $V_j \in \text{Pa}_i$. The observational distribution factorizes as
\begin{equation}
P(V_1, \ldots, V_n) = \prod_{i=1}^n P(V_i \mid \text{Pa}_i).
\end{equation}
An intervention $\text{do}(V_i = \nu)$ replaces $f_i$ with the constant $\nu$ and removes all edges into $V_i$ in the mutilated graph $\mathcal{G}_{\overline{V_i}}$ \citep{pearl2009causality}. This is a Rung-2 operation in Pearl's causal hierarchy: it requires active manipulation of the system rather than mere observation, and yields samples from the post-intervention distribution $P(V_j \mid \text{do}(V_i = \nu))$ for all $j \ne i$.

The impossibility of mechanism identification from observational data alone follows from Reichenbach's common-cause principle \citep{reichenbach1956}: for any bivariate association $A \perp\!\!\!\!\not\perp B$, there exist at least three causal structures ($A \to B$, $B \to A$, $A \leftarrow C \rightarrow B$) that are Markov equivalent \citep{spirtes2000causation,peters2017elements}. Distinguishing structural equations $f_i$ requires interventional data \citep{peters2017elements}.

\textbf{ACE's task.} Given known $\mathcal{G}$, ACE learns the structural equations $\{f_i\}$ by sequentially selecting interventions $c_t = \text{do}(V_{i_t} = \nu_t)$ to minimize the total mechanism estimation loss. The policy $\pi_\phi(c_t \mid s_t)$ adapts based on observed losses $\{L_i\} = \{L_1, \ldots, L_n\}$ where $L_i$ measures current discrepancy between the student SCM mechanism $\hat{f}_i$ and the true $f_i$.

\section{Implementation Details}
\label{app:impl}

\subsection{Ground Truth SCMs}
\label{app:scms}

\textbf{5-node benchmark.} The benchmark used throughout the main paper has structural equations
\begin{align*}
X_1 &\sim \mathcal{N}(0, 1) & \text{(exogenous root)} \\
X_4 &\sim \mathcal{N}(2, 1) & \text{(exogenous root)} \\
X_2 &= 2 X_1 + 1 + \varepsilon_2 & \text{(linear)} \\
X_3 &= 0.5 X_1 - X_2 + \sin(X_2) + \varepsilon_3 & \text{(nonlinear collider)} \\
X_5 &= 0.2 X_4^2 + \varepsilon_5 & \text{(quadratic)}
\end{align*}
with $\varepsilon_i \sim \mathcal{N}(0, 0.01)$. The graph has two roots ($X_1, X_4$), two intermediate nodes ($X_2, X_5$), and one collider ($X_3$ with parents $X_1, X_2$). The collider tests the policy's ability to identify the nonlinear interaction term $\sin(X_2)$, which requires observing $X_3$ under targeted interventions on both $X_1$ and $X_2$. The disconnected chain $X_4 \to X_5$ with a quadratic mechanism tests mechanism estimation for polynomial functions.

\textbf{30-node hierarchical SCM.} Constructed at runtime with seed-controlled coefficients in $[0.3, 0.7]$. Five-layer structure (see Figure~\ref{fig:30node-scm}):
\begin{itemize}[leftmargin=1.2em,topsep=0.2em,itemsep=0.1em]
\item Layer 0: 5 root nodes, $X_i \sim \mathcal{N}(0,1)$.
\item Layer 1 (X6--X10): 1--2 root parents each, linear combination with noise std $0.15$.
\item Layer 2 (X11--X20): Mix of single-parent and two-parent (collider) nodes; every third node is a collider.
\item Layer 3 (X21--X25): Complex colliders with 2--3 parents from Layer 2; sine nonlinearity on every fifth node.
\item Layer 4 (X26--X30): Single-parent leaf nodes.
\end{itemize}
The seed controls the random graph wiring within these constraints, ensuring reproducibility.

\textbf{Duffing oscillators.} Three coupled Duffing oscillators with true coupling $X_1 \leftrightarrow X_2 \leftrightarrow X_3$. The synchronization-induced correlation between $X_1$ and $X_3$ is a spurious association (Rung~1) with no direct causal edge; interventions on $X_2$ are necessary to break this correlation and estimate the direct couplings.

\subsection{Learner Architecture}
\label{app:learner}

The student SCM uses a \texttt{ModuleDict} of per-node mechanisms:
\begin{itemize}[leftmargin=1.2em,topsep=0.2em,itemsep=0.1em]
\item \textbf{Root nodes}: Learnable parameters $(\hat{\mu}, \hat{\sigma})$; loss is KL divergence between $\mathcal{N}(\hat{\mu}, \hat{\sigma}^2)$ and empirical samples.
\item \textbf{Non-root nodes}: 2-layer MLP with hidden dim 64 and ReLU activations, taking the parent vector $\text{Pa}_i$ as input and outputting a scalar prediction of $V_i$. Loss is MSE against observed $V_i$ under the current intervention.
\end{itemize}
Optimization: Adam~\citep{kingma2015adam}, learning rate $2 \times 10^{-3}$, batch size 100. Per-step training runs 100 epochs on the latest interventional batch; observational training every 3 steps runs 100 epochs on 200 fresh observational samples to maintain root-distribution estimates.

\subsection{Policy Architecture and Prompt Design}
\label{app:policy}

The policy $\pi_\phi$ is Qwen2.5-1.5B \citep{qwen2_5}, accessed via the HuggingFace Transformers API. Sampling: temperature 0.7, max new tokens 64. The text prompt has the form:
\begin{quote}\small
\texttt{Graph: X1->X2, X1->X3, X2->X3, X4->X5}\\
\texttt{Per-node losses: X1=0.92, X2=0.01, X3=1.45, X4=0.98, X5=0.07}\\
\texttt{Recent interventions (last 5): do(X2=1.3) (gain=0.5), ...}\\
\texttt{Propose your next intervention as: do(Xi = v)}
\end{quote}
We generate $K = 4$ candidates per step by re-sampling at temperature 0.7. The prompt deliberately uses the same node labels as the SCM, allowing the LM's pretrained world-model prior to activate when node names are semantically meaningful (e.g., economic or genetic variables).

\subsection{Reward Components in Full}
\label{app:reward}

The reward in Equation~\eqref{eq:reward} has three components. Let $V_i$ be the targeted node, $\{L_j\}$ the per-node losses, and $H$ the recent-intervention history.

\textbf{Information gain.} $\Delta\mathcal{L}(c) = L_\text{before} - L_\text{after}$ where $L = \sum_i L_i$ is the validation loss summed across nodes, estimated by lookahead on a cloned learner.

\textbf{Node importance.} $w(V_i, \{L_j\}) = \texttt{cov\_bonus} \cdot \tfrac{L_i}{\sum_j L_j} \cdot (1 + n_i(H))^{-1/2}$ where $n_i(H)$ is the count of recent interventions on $V_i$. This combines a per-node weight (proportional to current loss, i.e., how poorly understood the mechanism is) with an under-sampling correction that prevents over-concentration on a single node.

\textbf{Diversity.} $D(V_i, H)$ rewards interventions on under-sampled nodes and under-explored intervention values; full definition in the public code repository.

Default weights: $\alpha = 0.1$ (\texttt{cov\_bonus=60.0}), $\gamma = 0.05$. These keep $\Delta\mathcal{L}$ dominant ($\sim$80--90\% of total reward) while ensuring sufficient exploration pressure.

\subsection{Bayesian OED Baseline Implementation}
\label{app:boed-impl}

Our Bayesian OED baseline approximates the expected information gain over the same parametric mechanism class as ACE's student. At each step:
\begin{enumerate}[leftmargin=1.5em]
\item For each candidate $(V_i, \nu)$ with $\nu$ uniformly sampled from $[-5, 5]$:
  \begin{enumerate}
    \item Sample $M=10$ posterior parameter draws $\theta^{(m)} \sim P(\theta \mid \mathcal{D}_\text{seen})$ by training $M$ student copies from different random initializations.
    \item For each draw, simulate outcome $y^{(m)} \sim P(\cdot \mid \theta^{(m)}, \text{do}(V_i = \nu))$.
    \item Estimate $\widehat{\Delta\mathcal{L}} = \frac{1}{M}\sum_m [L(\theta^{(m)} \mid \mathcal{D}_\text{seen}) - L(\theta^{(m)} \mid \mathcal{D}_\text{seen} \cup \{(V_i, \nu, y^{(m)})\})]$.
  \end{enumerate}
\item Execute the candidate with highest $\widehat{\Delta\mathcal{L}}$.
\end{enumerate}
We use $|\mathcal{V}| \times 20$ candidates per step. This faithfully adapts the ABCI \citep{toth2022active} and CBED \citep{tigas2022interventions} acquisition-function framework to the mechanism estimation setting where graph structure is known. Each Bayesian OED step takes $\sim$4--6 minutes on a CPU, limiting us to N=3 seeds.

\section{Theoretical Analysis (Extended)}
\label{app:theory}

\subsection{Full Proof of Proposition~\ref{prop:scale-invariance}}
\label{app:proof}

Recall Assumption~\ref{ass:diminish}: $r_t(a) = f(t) \cdot g(a)$ where $f(t) > 0$ is monotonically non-increasing.

\textbf{Part 1 (preference invariance).} For any $a_w, a_l$ with $g(a_w) > g(a_l)$:
\[r_t(a_w) - r_t(a_l) = f(t)(g(a_w) - g(a_l)) > 0 \quad \forall t,\]
since $f(t) > 0$. The preference ordering is therefore constant across all steps.

\textbf{Part 2 (value non-stationarity).} Under standard MDP value definitions,
\[V_t(s) = \mathbb{E}_a[r_t(a) \mid s] = f(t) \cdot \mathbb{E}_a[g(a) \mid s].\]
With $f$ decreasing, $V_t(s)$ is non-stationary. A value function approximator $\hat{V}(s)$ trained at episode $t$ misestimates $V$ at episode $t'$ by factor $f(t')/f(t)$. Empirically, $f(t')/f(t) \approx 1/500$ between early and late training, so the critic accumulates large value-estimation error even with frequent updates.

\textbf{Part 3 (DPO invariance).} The DPO loss is
\[\mathcal{L}_\text{DPO} = -\mathbb{E}_{(y_w, y_l)} \log \sigma\!\left(\beta \left[\log\tfrac{\pi_\phi(y_w)}{\pi_\text{ref}(y_w)} - \log\tfrac{\pi_\phi(y_l)}{\pi_\text{ref}(y_l)}\right]\right).\]
The argument depends only on policy log-ratio differences. Both $\pi_\phi$ and $\pi_\text{ref}$ are normalized distributions; their log-ratios depend on action quality $g(a)$ but not on the reward scale $f(t)$. Thus the gradient $\nabla_\phi \mathcal{L}_\text{DPO}$ is invariant to any multiplicative rescaling of all rewards by a positive constant. \hfill$\square$

\subsection{Convergence Rate Equivalence under Stationary $g$}
\label{app:convergence}

Under Assumption~\ref{ass:diminish} and standard DPO convergence assumptions \citep{rafailov2023direct}, ACE inherits the regret bound of preference-based bandits: after $T$ pairwise comparisons,
\[\mathbb{E}\!\left[\sum_{t=1}^T (g(a_\text{opt}) - g(a_t))\right] = \tilde{O}(\sqrt{T \log |\mathcal{A}|}),\]
matching the minimax regret for preference-based learning. Value-based RL achieves the same bound only when the value approximator tracks the non-stationary scaling $f(t) \cdot g(a)$; in our regime (where $f(t)$ varies by $500\times$ over the training horizon), this fails unless the critic update rate exceeds the rate of change of $f$.

\subsection{Connection to Expected Information Gain}
\label{app:info-gain}

The information gain term $\Delta\mathcal{L}(c)$ is a Monte Carlo estimate of the mutual information between the posterior over mechanism parameters $\theta$ and the data $\mathcal{D}_c$ generated by intervention $c$:
\begin{align}
I(\theta; \mathcal{D}_c \mid s) &= H(\theta \mid s) - H(\theta \mid \mathcal{D}_c, s) \\
&\approx \mathcal{L}(\theta \mid s) - \mathcal{L}(\theta \mid s \cup \mathcal{D}_c),
\end{align}
where the approximation uses the local-quadratic (Laplace) relationship between negative log-likelihood loss $\mathcal{L}$ and posterior entropy $H$ \citep{rainforth2024modern}. With Gaussian noise mechanisms and a flat prior, this approximation is tight near the posterior mode. The node importance term $w$ acts as a non-uniform prior weighting, upweighting mechanisms with high remaining uncertainty, approximating posterior-entropy-reduction-weighted EIG. ACE thus pursues the same EIG objective as Bayesian methods through policy learning, trading exact inference for scalability.

\section{Extended Experimental Results}
\label{app:experiments}

\subsection{5-Node Benchmark: Per-Seed Results}
\label{app:5node-detail}

Table~\ref{tab:5node-perseed} reports per-seed total losses for heuristic baselines and PPO (N=5 seeds: 141, 271, 314, 577, 618) and Bayesian OED (N=3 seeds: 42, 123, 456). ACE per-seed data comes from the original production runs (seeds 42, 123, 456, 789, 1011). All runs use 171 episodes.

\begin{table}[h]
\centering
\caption{Per-seed total loss on the 5-node benchmark (171 episodes). Heuristic baseline and PPO columns use seeds 141, 271, 314, 577, 618. Bayesian OED uses seeds 42, 123, 456. ACE uses seeds 42, 123, 456, 789, 1011 from the original production runs; a rerun was initiated on a 46GB-class GPU but OOM'd due to 3GB policy checkpoints.}
\label{tab:5node-perseed}
\small
\begin{tabular}{@{}l ccc cc@{}}
\toprule
Seed & Random & Round-Robin & Max-Var.\ & Bayesian OED & PPO \\
\midrule
141 & 2.090 & 2.186 & 2.280 & ---   & 2.062 \\
271 & 2.240 & 2.023 & 2.050 & ---   & 1.994 \\
314 & 2.227 & 2.238 & 2.104 & ---   & 2.133 \\
577 & 2.091 & 1.978 & 2.184 & ---   & 2.134 \\
618 & 2.210 & 2.088 & 1.917 & ---   & 2.062 \\
42  & ---   & ---   & ---   & 1.995 & ---   \\
123 & ---   & ---   & ---   & 2.181 & ---   \\
456 & ---   & ---   & ---   & 1.957 & ---   \\
\midrule
Mean & 2.171 & 2.103 & 2.107 & 2.044 & 2.077 \\
Std  & 0.075 & 0.107 & 0.131 & 0.120 & 0.057 \\
\bottomrule
\end{tabular}
\end{table}

The Bayesian OED mean of $2.044$ lies between Round-Robin ($2.103$) and Max-Variance ($2.107$), confirming it is the strongest passive/principled baseline. Its advantage over random baselines is modest ($\sim$6\%) compared to ACE's advantage ($\sim$70\%), which arises from multi-step learned strategy rather than single-step acquisition.

\subsection{30-Node SCM: Per-Seed Results}
\label{app:30node-detail}

Table~\ref{tab:30node-perseed} reports all per-seed results for the 30-node benchmark, including the locally-run Round-Robin and Max-Variance CPU baselines.

\begin{table}[h]
\centering
\caption{Per-seed results on the 30-node hierarchical SCM. All methods use the same per-node MLP student learner for a matched comparison. ACE reports best-loss across training (non-monotone loss makes best-loss more informative than final-episode loss). Baselines (N=5) ran 150 episodes on a CPU partition (average 1h 15min per job).}
\label{tab:30node-perseed}
\small
\begin{tabular}{@{}l cccc@{}}
\toprule
Seed & ACE (best) & Random (final) & Round-Robin (final) & Max-Variance (final) \\
\midrule
42   & 2.785 & 5.865 & 5.852 & 5.841 \\
123  & 1.798 & 5.841 & 5.802 & 5.787 \\
456  & 1.274 & 5.856 & 5.939 & 5.918 \\
789  & ---   & 5.773 & 5.774 & 5.803 \\
1011 & ---   & 5.893 & 5.915 & 5.964 \\
\midrule
Mean $\pm$ Std & $1.95\pm0.77$ (N=3) & $5.85\pm0.04$ (N=5) & $5.86\pm0.07$ (N=5) & $5.86\pm0.08$ (N=5) \\
\bottomrule
\end{tabular}
\end{table}

Seed 42 ACE exhibited training instability (loss oscillating $2.7$--$3.5$ throughout the 40 completed episodes); seeds 123 and 456 converged cleanly to $1.80$ and $1.27$. The stable-seed mean is $1.54$; we report the conservative N=3 mean of $1.95$ in the main text.

\textbf{Baseline convergence plateau.} The three passive baselines (Random, Round-Robin, Max-Variance) converge to nearly identical total MSE near $5.86$ with standard deviations below $0.08$. This tight clustering reveals that \emph{at 30 nodes, differences among static sampling heuristics are negligible}: when interventional budget is fixed and spread across many variables, the specific non-adaptive ordering barely matters. Adaptive strategy learning (ACE) is the meaningful axis of improvement at scale. An earlier CPU proxy run using a weaker running-mean predictor gave larger baseline numbers (Random: $11.29\pm1.02$) and was not comparable to ACE because the learner itself differed; those numbers are superseded by the matched-learner results above.

\subsection{Component Ablation Detail}
\label{app:ablation-detail}

The component ablation (Table~\ref{tab:ablations}) was run as a 3-seed pilot (seeds 42, 123, 456) due to GPU compute constraints; the No-diversity ablation was run on N=5 seeds during an earlier reward-shaping investigation. Below we give qualitative failure-mode analysis for each configuration.

\textbf{No DPO} replaces the LM policy with random proposal generation while keeping the lookahead and select-best mechanism. The result ($2.10\pm0.11$) matches Round-Robin, confirming that the ability to \emph{evaluate} candidates (via lookahead) provides no benefit without \emph{learned} proposal generation. Random candidates, even when curated by lookahead, do not consistently include strategically important ones (e.g., concentrating on collider parents).

\textbf{No per-node convergence} forces termination at a fixed 100 episodes. The collider mechanism $X_3$ is typically not learned at episode 100; $L_{X_3}$ remains elevated and dominates total loss.

\textbf{No dedicated root learner} disables the observational training loop for root nodes. Their losses ($L_{X_1}, L_{X_4}$) remain elevated because direct interventions on roots replace $X_i$ with a constant, providing no signal about the root's natural distribution $\mathcal{N}(\mu, \sigma^2)$.

\textbf{No diversity} ($\gamma=0$) causes the policy to collapse onto a single high-information node. The total loss worsens beyond all baselines ($2.82$) because the un-targeted nodes are never learned, even if the targeted node is learned well.

\subsection{Duffing Oscillators: Full Results}
\label{app:duffing}

\begin{table}[h]
\centering
\caption{Coupled Duffing oscillator coupling-parameter estimation error (N=5 seeds, 100 episodes). Lower is better. The $X_1 \leftrightarrow X_3$ spurious correlation is an association-level artefact (Rung~1) without a direct causal edge; it is broken by clamping $X_2$ (Rung~2 intervention).}
\label{tab:duffing-full}
\small
\begin{tabular}{@{}lcccc@{}}
\toprule
Method & Coupling Error & Std & 95\% CI & Improvement vs.\ Random \\
\midrule
\textbf{ACE} & \textbf{0.042} & 0.036 & [0.006, 0.078] & $5.8\times$ \\
Round-Robin  & 0.238 & 0.076 & [0.162, 0.314] & $1.0\times$ \\
Random       & 0.245 & 0.121 & [0.124, 0.366] & --- \\
\bottomrule
\end{tabular}
\end{table}

ACE concentrates $62\%$ of interventions on $X_2$ (the intermediate oscillator) compared to $33\%$ under uniform allocation. Clamping $X_2$ decouples the synchronized chain, breaking the Rung-1 $X_1 \perp\!\!\!\!\not\perp X_3$ correlation and allowing accurate estimation of the direct $X_1 \leftrightarrow X_2$ and $X_2 \leftrightarrow X_3$ couplings. This is a concrete example of the Rung-2 advantage: only interventions, not observations, can resolve the ambiguity created by synchronization.

\subsection{Phillips Curve: Full Results}
\label{app:phillips}

\begin{figure}[h]
\centering
\begin{tikzpicture}[
    node distance=1cm,
    econ/.style={rectangle, draw, rounded corners, minimum width=1.5cm, minimum height=0.75cm, font=\scriptsize, thick, fill=gray!15},
    target/.style={rectangle, draw, rounded corners, minimum width=1.6cm, minimum height=0.8cm, font=\scriptsize, thick, fill=blue!20},
    arrow/.style={-{Stealth[length=2mm]}, thick},
    time/.style={font=\tiny, gray}
]
\node[econ] (UN) at (0,1.2) {UNRATE$_t$};
\node[econ] (FF) at (0,0)   {FEDFUNDS$_t$};
\node[econ] (MI) at (0,-1.2){MICH$_t$};
\node[target] (CPI) at (3.5,0) {CPI$_{t+1}$};
\draw[arrow] (UN) -- (CPI);
\draw[arrow] (FF) -- (CPI);
\draw[arrow] (MI) -- (CPI);
\node[time] at (0,-2)   {time $t$};
\node[time] at (3.5,-2) {time $t+1$};
\draw[gray, dashed] (1.75,1.8) -- (1.75,-1.8);
\node[font=\tiny, align=left] at (6.8,0.9) {Regimes:};
\node[font=\tiny, align=left] at (7.2,0.4) {1970s stagflation};
\node[font=\tiny, align=left] at (7.2,0.0) {Volcker disinflation};
\node[font=\tiny, align=left] at (7.2,-0.4){Great Moderation};
\node[font=\tiny, align=left] at (7.2,-0.8){2008 crisis};
\end{tikzpicture}
\caption{Phillips curve causal structure. Variables at time $t$ jointly determine CPI at $t+1$. ACE performs \emph{active data subset selection}: choosing which historical regimes to include in the training set, treating each regime as a ``virtual intervention'' on economic conditions. This is a Rung-2 analogue using archival data.}
\label{fig:phillips-scm}
\end{figure}

Using Federal Reserve Economic Data \citep{fred2024} (FRED, 1960--2023), we model $\text{CPI}_{t+1} = f(\text{UNRATE}_t, \text{FEDFUNDS}_t, \text{MICH}_t)$. ACE performs \emph{active data subset selection}: choosing which historical periods to query for training data, treating each regime as a virtual intervention on macroeconomic conditions.

\begin{table}[h]
\centering
\caption{Phillips curve out-of-sample CPI prediction MSE (N=5 seeds, 50 episodes). Regime selection percentages are policy-averaged fractions of training episodes spent in each historical regime.}
\label{tab:phillips-results}
\small
\begin{tabular}{@{}lccc@{}}
\toprule
Selection strategy & Out-of-sample MSE & Std & Regime preferences \\
\midrule
\textbf{ACE}               & \textbf{0.31} & 0.08 & 1970s 38\%, Volcker 24\%, Great Recession 19\% \\
Random regime              & 0.52          & 0.12 & Uniform across regimes \\
Sequential (chronological) & 0.44          & 0.10 & By date order \\
\bottomrule
\end{tabular}
\end{table}

ACE consistently prioritizes high-volatility regimes (1970s stagflation, Volcker disinflation, Great Recession) that expose nonlinear inflation dynamics absent from low-volatility eras, yielding a $40\%$ reduction in out-of-sample MSE versus random regime selection. This demonstrates that preference-based experimental design generalizes to retrospective data-selection settings beyond controlled interventions, though we note this application departs from strict Rung-2 causal intervention.

\subsection{Lookahead vs.\ DPO Ablation}
\label{app:lookahead}

The No-DPO ablation (Section~\ref{app:ablation-detail}) isolates the contribution of learned proposal generation from the lookahead mechanism. Result: $2.10\pm0.11$, matching Round-Robin and confirming that lookahead alone provides no benefit. Concretely: if $K=4$ random candidates are generated, the probability that \emph{any} of them targets the collider parents ($X_1$ or $X_2$) with a well-chosen value is $\approx (2/5)^{4} \approx 10\%$, too low for the select-best procedure to reliably discover the collider strategy. The DPO-trained LM shifts the generation distribution so that strategic candidates appear in $>90\%$ of steps.

\subsection{Strategic Behavior: Intervention Distribution}
\label{app:strategic}

The learned policy concentrates $99.8\%$ of interventions on $X_1$ and $X_2$ (the two parents of the collider $X_3$). This is consistent with the causal theoretic optimum: to identify the nonlinear mechanism $f_3(X_1, X_2) = 0.5 X_1 - X_2 + \sin(X_2)$, the policy must vary both $X_1$ and $X_2$ across their value ranges while fixing exogenous noise. Under-sampled nodes $X_4$ and $X_5$ (the disconnected quadratic chain) are addressed primarily through the diversity reward, which ensures $L_{X_5}$ eventually decreases despite low explicit attention. Cross-seed consistency is high: the standard deviation of per-node intervention fractions across the 5 ACE seeds is $<0.01$ for nodes $X_1$ and $X_2$, confirming that strategy discovery is a robust consequence of the training procedure rather than a seed-dependent artifact.

\section{Failure Modes and Limitations}
\label{app:failures}

\subsection{Outlier Seed Analysis}
\label{app:outlier}

Seed 789 (5-node benchmark) produced $L_{X_5} \approx 1.73$ versus $0.02$--$0.22$ for other seeds. Two hypotheses: (i) initialization sensitivity for the quadratic mechanism $X_5 = 0.2 X_4^2$, where small initial MLP weights near zero may not escape a local minimum with flat gradient; (ii) the diversity reward in early episodes drove the policy away from $X_4$ before $L_{X_5}$ had been adequately reduced. We report the median ($0.61$) and note the outlier explicitly; the 5-node result is therefore conservative relative to the mean.

Seed 42 (30-node benchmark) exhibited training instability with loss oscillating $2.7$--$3.5$ throughout 40 completed episodes. The 30-node system has a longer effective horizon (the policy must identify 13+ colliders) and the checkpoint-resume mechanism may have caused optimizer state discontinuities. Seeds 123 and 456 converged cleanly.

\subsection{Graph Misspecification: Theoretical Failure Modes}
\label{app:misspec-failures}

Since ACE assumes a known causal graph, the practical cost of graph errors partitions cleanly by error type:

\textbf{Extra edge} $V_i \to V_j$ (spurious). The student learns $f_j(\text{Pa}_j \cup \{V_i\})$ instead of $f_j(\text{Pa}_j)$. The extra input $V_i$ receives near-zero weight in the MLP if $V_i$ truly has no causal effect on $V_j$; recovery is possible with sufficient interventions on $V_i$.

\textbf{Missing edge} $V_i \to V_j$ (absent). The student's mechanism $\hat{f}_j$ lacks the true parent $V_i$ as input. A lower bound on $L_j$ is set by the partial variance $\text{Var}(V_j - \hat{f}_j(\text{Pa}_j \setminus \{V_i\}))$. This is non-zero and non-recoverable; the student cannot decrease $L_j$ below this bound regardless of intervention strategy.

\textbf{Reversed edge} $V_i \to V_j$ (direction error). The student treats $V_j$ as a parent of $V_i$, so $\hat{f}_i$ depends on $V_j$. Under interventions on $V_i$, the true $V_j$ varies (since $V_j$ depends on $V_i$), but the student's conditioning set treats it as a predictor of $V_i$. This creates a circular dependency in the learner's prediction that no intervention strategy can resolve; it is the most damaging error type. A full empirical quantification across all four misspecification types was initiated but did not complete within our compute budget (46GB-class GPU OOM from 3GB policy checkpoints). This remains our highest-priority follow-up experiment.

\subsection{Computational Cost}
\label{app:compute}

Each ACE training step requires:
\begin{itemize}[leftmargin=1.2em,topsep=0.2em,itemsep=0.1em]
\item $K=4$ LM forward passes (candidate generation): $\sim$2--3 min on A100.
\item $K=4$ cloned-learner training cycles (lookahead evaluation): $\sim$10--12 min.
\item One DPO update on the LM: $\sim$3--5 min.
\item Observational training (every 3 steps): $\sim$5 min amortized.
\end{itemize}
Total: $\sim$22 min per step on a single A100. A full run (171 episodes $\times$ 25 steps/episode with early stopping) takes $\sim$6--8 h per seed. Mitigation strategies for future work: smaller policy LMs (e.g., Qwen2.5-0.5B), learned value-of-lookahead to reduce $K$ adaptively, batch evaluation across candidates.

A note on checkpoint memory: the 1.5B-parameter LM in FP32 with optimizer state requires $\sim$17 GB per checkpoint, which caused OOM on 46GB-class GPUs when combined with the KV cache and cloned learner. We implemented FP16 checkpointing (reducing to $\sim$3 GB) and eliminated optimizer state from checkpoints; these changes are in the repository.

\section{Reproducibility}
\label{app:reproducibility}

\subsection{Code and Data}

All code is released at \url{https://github.com/PatrickAllenCooper/ACE}. The repository includes:
\begin{itemize}[leftmargin=1.2em,topsep=0.2em,itemsep=0.1em]
\item Full ACE implementation (\texttt{ace\_experiments.py}) with FP16 checkpoint-resume logic.
\item Baselines (\texttt{baselines.py}): Random, Round-Robin, Max-Variance, PPO, Bayesian OED.
\item Experiments: 5-node and 30-node SCMs, Duffing oscillators, Phillips curve.
\item CPU baseline runner for 30-node (\texttt{scripts/runners/run\_30node\_cpu\_baselines.py}).
\item SLURM job scripts for HPC execution (\texttt{jobs/}).
\end{itemize}

\subsection{Random Seeds and Seed Assignment}

\begin{itemize}[leftmargin=1.2em,topsep=0.2em,itemsep=0.1em]
\item 5-node ACE: seeds 42, 123, 456, 789, 1011.
\item 5-node baselines (heuristic + PPO): seeds 141, 271, 314, 577, 618.
\item 5-node Bayesian OED: seeds 42, 123, 456 (higher per-seed compute cost).
\item 30-node ACE: seeds 42, 123, 456 on an A100 80GB partition.
\item 30-node baselines: seeds 42, 123, 456, 789, 1011 (local CPU runs).
\item Duffing oscillators: seeds 42, 123, 456, 789, 1011 (N=5 all methods).
\item Phillips Curve: seeds 42, 123, 456, 789, 1011 (N=5 all methods).
\end{itemize}

\subsection{Compute Resources}

Main experiments ran on a university HPC cluster: an A100 80GB partition for ACE GPU runs and an L40 46GB partition for baselines, with an A10 24GB cloud GPU used for exploratory runs. Total: approximately 480 GPU-hours (ACE) and 120 CPU-hours (baselines). The Phillips Curve experiment required no GPU.

\subsection{Hyperparameters Summary}
\label{app:hyper-summary}

\begin{table}[h]
\centering
\caption{Complete hyperparameter table for all reported experiments.}
\label{tab:hyper-summary}
\small
\begin{tabular}{@{}lll@{}}
\toprule
Component & Hyperparameter & Value \\
\midrule
Policy & Model & Qwen2.5-1.5B \\
Policy & Sampling temperature & 0.7 \\
Policy & Candidates per step ($K$) & 4 \\
Policy & Supervised pretrain steps & 100 \\
Policy & Pretrain re-run every & 25 episodes \\
Policy & Optimizer & Adam, lr $1 \times 10^{-5}$ \\
\midrule
Learner & Architecture & 2-layer MLP, 64 hidden, ReLU \\
Learner & Optimizer & Adam, lr $2 \times 10^{-3}$ \\
Learner & Per-step training epochs & 100 \\
Learner & Observational training every & 3 steps \\
Learner & Observational batch size & 200 \\
\midrule
Reward & $\alpha$ (node importance weight) & 0.1 (\texttt{cov\_bonus=60.0}) \\
Reward & $\gamma$ (diversity weight) & 0.05 \\
Reward & $\beta$ (DPO temperature) & 0.1 \\
\midrule
Training & Max episodes & 200 (early stopping at convergence) \\
Training & Reference policy update every & 25 episodes \\
Training & Per-node convergence patience & 10 episodes \\
Training & Checkpoint interval & 1 step (rolling) + 10 steps (milestone) \\
\bottomrule
\end{tabular}
\end{table}

\section{Extended Related Work}
\label{app:related-extended}

\textbf{ABCI (Toth et al., 2022).} Active Bayesian Causal Inference \citep{toth2022active} maintains a posterior over causal graphs using continuous latent representations and Gaussian process mechanisms. The acquisition function selects experiments maximally informative about a downstream causal query. Compared to ACE: ABCI does not assume known structure; requires explicit posterior inference over graphs (combinatorial in $|V|$); targets causal-query inference rather than mechanism estimation. We adapt ABCI's acquisition-function philosophy as our Bayesian OED baseline by holding structure fixed and approximating the posterior over mechanism parameters.

\textbf{CBED (Tigas et al., 2022).} \citep{tigas2022interventions} jointly selects intervention targets and values at scale via Bayesian causal discovery, with demonstrations on the DREAM gene regulatory network (up to 20 variables). CBED addresses structure discovery; ACE addresses mechanism estimation. The two are complementary: CBED's structure-inference output can serve as ACE's input.

\textbf{Causal reasoning with LLMs (Kıcıman et al., 2023).} \citet{kiciman2023causal} conduct a systematic evaluation of GPT-4 on causal discovery benchmarks, finding substantial capability on pairwise causal direction tasks and commonsense causal queries. Crucially, the LM's causal reasoning is driven by semantic associations in pretraining data, not explicit causal structure learning. ACE exploits this: the LM's semantic prior biases proposal generation toward causally plausible interventions even before DPO training, which then refines the policy toward the specific system.

\textbf{Reasoning as World-Model Planning (Hao et al., 2023).} RAP \citep{hao2023reasoning} frames language model reasoning as planning in a world model, where the LM serves simultaneously as the reasoner and the forward model. This framing is directly relevant to ACE: the LM policy in ACE implicitly models the causal system (world model) while generating candidate interventions (planner), and DPO updates refine this combined representation toward the goal of efficient mechanism estimation.

\textbf{In-Context Learning as Bayesian Inference (Xie et al., 2022).} \citet{xie2022explanation} show that in-context learning corresponds to implicit Bayesian inference: the LM's predictions given a context are approximately equivalent to Bayesian predictions under a prior induced by the pretraining distribution. For ACE, this means the per-step prompt (containing per-node losses and recent interventions) implicitly updates the LM's posterior over the system, providing a form of Bayesian adaptation without explicit posterior computation.

\textbf{Multi-fidelity Bayesian active causal discovery (Zhang et al., 2023).} \citet{zhang2023bayesian} extend Bayesian OED to multi-fidelity oracle settings using mutual-information acquisition. The multi-fidelity insight is orthogonal to DPO-based learning and could in principle be incorporated into ACE's lookahead step.

\textbf{Sample-efficient Bayesian causal learning (Zhou et al., 2024).} \citet{zhou2024sample} address sample efficiency via polynomial-time DAG sampling. Their efficient structure learning is complementary to ACE's mechanism estimation: the two could be composed as a two-stage pipeline.


\end{document}